\documentclass[english,a4paper,12pt]{article}

\usepackage{amsxtra, amsfonts, amssymb, amstext, amsmath, mathtools}
\usepackage{amsthm}
\usepackage{booktabs}
\usepackage{nicefrac}
\usepackage{xspace}
\usepackage{url}
\usepackage{graphics,color}
\usepackage[algo2e,ruled,vlined,linesnumbered]{algorithm2e}\SetArgSty{upshape}
\usepackage{algcompatible}
\usepackage{algorithm}
\usepackage{wrapfig}

\usepackage[T1]{fontenc}

\newtheorem{theorem}{Theorem}
\newtheorem{lemma}[theorem]{Lemma}
\newtheorem{corollary}[theorem]{Corollary}

\newcommand{\NSGA}{\mbox{NSGA}\nobreakdash-II\xspace}

\newcommand{\NSGAT}{\mbox{NSGA}\nobreakdash-III\xspace}
\newcommand{\NSGAthree}{\mbox{NSGA}\nobreakdash-III\xspace}
\newcommand{\NSGAD}{\mbox{NSGA}\nobreakdash-II-T\xspace}
\newcommand{\SMS}{\mbox{SMS-EMOA}\xspace}
\DeclareMathOperator{\tcd}{tCD}
\DeclareMathOperator{\dCD}{\tcd}
\DeclareMathOperator{\cDis}{cDis}

\newcommand{\omm}{\textsc{OneMinMax}\xspace}
\newcommand{\momm}{m\textsc{OneMinMax}\xspace}

\newcommand{\cocz}{\textsc{COCZ}\xspace}

\newcommand{\lotz}{\textsc{LOTZ}\xspace}
\newcommand{\ojzj}{\textsc{OneJumpZeroJump}\xspace}
\newcommand{\mojzj}{m\textsc{OJZJ}\xspace}
\newcommand{\mcocz}{m\textsc{COCZ}\xspace}
\newcommand{\mlotz}{m\textsc{LOTZ}\xspace}
\newcommand{\DLTB}{\textsc{DLTB}\xspace}

\newcommand{\mei}{\textsc{MEI}\xspace}

\DeclareMathOperator{\polylog}{polylog}

\newcommand{\N}{\ensuremath{\mathbb{N}}} 

\let\originalleft\left
\let\originalright\right
\renewcommand{\left}{\mathopen{}\mathclose\bgroup\originalleft}
\renewcommand{\right}{\aftergroup\egroup\originalright}

\usepackage{hyperref}

\begin{document}
\title{A Crowding Distance That Provably Solves the Difficulties of the NSGA-II in Many-Objective Optimization}

\author{Weijie Zheng\thanks{School of Computer Science and Technology, International Research Institute for Artificial Intelligence, Harbin Institute of Technology, Shenzhen, China} \and Yan Gao\footnotemark[1] \and Benjamin Doerr\thanks{Laboratoire d'Informatique (LIX), CNRS, \'Ecole Polytechnique, Institut Polytechnique de Paris, Palaiseau, France}~\thanks{Corresponding author}}


%

\maketitle

\sloppy{

\begin{abstract}
Recent theoretical works have shown that the NSGA-II can have enormous difficulties to solve problems with more than two objectives. In contrast, algorithms like the NSGA-III or SMS-EMOA, differing from the NSGA-II only in the secondary selection criterion, provably perform well in these situations. 

To remedy this shortcoming of the NSGA-II, but at the same time keep the advantages of the widely accepted crowding distance, we use the insights of these previous work to define a variant of the crowding distance, called truthful crowding distance. Different from the classic crowding distance, it has for any number of objectives the desirable property that a small crowding distance value indicates that some other solution has a similar objective vector.

Building on this property, we conduct mathematical runtime analyses for the NSGA-II with truthful crowding distance. We show that this algorithm can solve the many-objective versions of the \omm, \cocz, \lotz, and $\ojzj_k$ problems in the same (polynomial) asymptotic runtimes as the NSGA-III and the SMS-EMOA. This contrasts the exponential lower bounds previously shown for the classic NSGA-II. For the bi-objective versions of these problems, our NSGA-II has a similar performance as the classic NSGA-II, gaining however from smaller admissible population sizes. For the bi-objective \omm problem, we also observe a (minimally) better performance in approximating the Pareto front.

These results suggest that our truthful version of the NSGA-II has the same good performance as the classic NSGA-II in two objectives, but can resolve the drastic problems in more than two objectives.

\end{abstract}

\section{Introduction}

In many practical applications, the problems to be solved have several, often conflicting objectives. Since such problems often do not have a single optimal solution, one resorts to computing a set of diverse good solutions (ideally so-called Pareto optima) and lets a human decision maker take the final decision among these.

One of the most successful algorithm for computing such a set of solutions for a multi-objective optimization problem is the \emph{Non-dominated Sorting Genetic Algorithm II (\NSGA)} by Deb et al.~\cite{DebPAM02}, currently cited more than 50,000 times according to Google scholar. 

While it was always known that the performance of this algorithm becomes weaker with increasing numbers of objectives -- this was the main motivation for Deb and Jain~\cite{DebJ14} to propose the \NSGAthree{}  --, very recent mathematical analyses of multi-objective evolutionary algorithms (MOEAs) could quantify and obtain a deeper understanding of this shortcoming. In~\cite{ZhengD23many}, it was proven that the \NSGA with any population size linear in the Pareto front size cannot optimize the simplistic \omm benchmark in subexponential time when the number of objectives is at least three (for two objectives, a small polynomial runtime guarantee was proven by Zheng and Doerr~\cite{ZhengD23aij}). In contrast, for the \NSGAthree and the \SMS, two algorithms differing from the \NSGA only in that the crowding distance is replaced by a different secondary selection criterion, polynomial runtime guarantees could be proven for the \omm and several other benchmarks in any (constant) number of objectives~\cite{WiethegerD23,OprisDNS24,ZhengD24,WiethegerD24}. This different optimization behavior suggests that it is the crowding distance which is the root for the problems of the \NSGA in higher numbers of objectives. 

Given that the \NSGA is the by far dominant MOEA in practice, clearly beating the \NSGAthree (cited less than 6,000 times according to Google scholar) and the \SMS (cited less than 2,200 times), and speculating that practitioners prefer working with a variant of the \NSGA rather than switching to a different algorithm (and also noting that the \NSGAthree and \SMS have some known shortcomings the \NSGA does not have), in this work we propose to use the \NSGA unchanged apart from a mild modification to the crowding distance. This change will again build on insights from Zheng and Doerr~\cite{ZhengD23many}. We defer the technical details to a separate section below 
and state here only that we call our crowding distance \emph{truthful crowding distance} since we feel that it better reflects how close a solution is to others.

Given that we build mostly on previous works of mathematical nature, and strongly profited from the precision of such results, we analyze the \NSGA with truthful crowding distance also via mathematical means. Our \textbf{main results} are the following. (i) For the standard many-objective benchmarks \momm, \mcocz, \mlotz, and \mojzj, the \NSGA with truthful crowding distance computed the whole Pareto front efficiently, in asymptotically the same time as the \NSGAthree or the \SMS. This demonstrates clearly that it was in fact a weakness of the original crowding distance that led to the drastic problems observed for the \NSGA in many-objective optimization. (ii) For the bi-objective versions of these benchmarks, for which the \NSGA was efficient, we show that the \NSGA with truthful crowding distance is equally efficient, and this already from population sizes on that equal the Pareto front of the problem, whereas the previous results needed a population size some constant factor larger. (iii) We also regard the problem of approximating the Pareto front when the population size is too small to cover the full Pareto front. Here we show that our \NSGA with sequential selection, in an analogous way as the sequential version of the classic \NSGA, computes good approximations to the Pareto front of the bi-objective \omm problem (the approximation quality is minimally better for our algorithm). 

In summary, these results show that the \NSGA with truthful crowding distance overcomes the difficulties of the classic \NSGA in many-objective optimization, but preserves its good performance in bi-objective optimization.

\section{Preliminaries}

In this work, we discuss variants of the \NSGA, the most prominent multi-objective evolutionary algorithm and one of the most successful approaches to solve multi-objective optimization problems. 

A \emph{multi-objective optimization problem} is a tuple $f = (f_1, \dots, f_m)$ of functions defined on a common \emph{search space}. As common in discrete evolutionary optimization, we always consider the search space $\{0,1\}^n$. We call $n$ the problem size. When using \emph{asymptotic notation}, this will be with respect to $n$ tending to infinity. 

Also always, our goal will be to \emph{maximize}~$f$. Since the individual objective $f_i$ might be conflicting, we usually do not have a single solution maximizing all $f_i$. In this case, the best we can hope for are solutions that are not strictly dominate by others. We say that a solution $x$ \emph{dominates} a solution $y$, written as $x \succeq y$, if $f_i(x) \ge f_i(y)$ for all $i \in [1..m]$. If in addition one of these inequalities is strict, we speak of \emph{strict domination}, denoted by $\succ$. The \emph{Pareto set} of a problem is the set of solutions that are not strictly dominated by another solution; the set of their objective values is called the \emph{Pareto front}. 

A common solution concept to multi-objective problems is to compute a small set $S$ of solutions such that $f(S)$ is the Pareto front or approximates it in some sense. The idea is that a human decision maker, based on preferences not included in the problem formulation, can then select from $S$ the final solution. 

For this problem, evolutionary algorithms have been employed with great success~\cite{CoelloLV07,ZhouQLZSZ11}. The by far dominant algorithm among these \emph{multi-objective evolutionary algorithms (MOEAs)} is the \emph{non-dominated sorting genetic algorithms II (\NSGA)} proposed by Deb et al.~\cite{DebPAM02}. This algorithm works with a population $P$, initialized randomly, of fixed size $N$. Each iteration of the main optimization loop consist of creating $N$ new solutions from these parents (``offspring'') and selecting from the combined parent and offspring population the next population of $N$ individuals.

Various ways of creating the offspring have been used. We shall regard random parent selection (each offspring is generated from randomly chosen parents) and fair parent selection (only with mutation, here from each parent and offspring is generated via mutation), 1-bit and bit-wise mutation with mutation rate~$1/n$, and uniform crossover. When using crossover, we assume that there is a positive constant $p<1$ (crossover rate) and in each iteration, with probability~$p$ and offspring is created via crossover, else via mutation. Binary tournament parent selection has also been studied, but the existing mathematical results, see, e.g., \cite{ZhengD23aij} suggest that it does not lead to substantially different results, but only to more complicated analyses. 

More important and characteristic for the \NSGA is the selection of the next parent population. The most important selection criterion is \emph{non-dominated sorting}, that is, the combined parent and offspring population $R$ is partitioned into fronts $F_1, F_2, \dots$ such that $F_i$ consist of all non-dominated elements (that is, elements not strictly dominated by another one) of $R \setminus (F_1 \cup \dots \cup F_{i-1})$. Individuals in an earlier front are preferred in the selection of the next population, that is, for the maximum $i^*$ such that $F_1 \cup \dots \cup F_{i^*-1}$ contains less than $N$ elements, these fronts all fully go into the next population. The remaining elements are selected from the \emph{critical front} $F_{i^*}$ using a secondary criterion, which is the crowding distance for the \NSGA. We defer the precise definition of the crowding distance to the subsequent section. 
We note that the non-dominated-sorting partition is uniquely defined and can be computed in time $O(m|R|^2)$. The crowding distance depends on how ties in sortings are broken, the crowding distance of all individuals in $F_{i^*}$ can be computed very efficiently in time $O(m |F_{i^*}| \log |F_{i^*}|)$.

The pseudocode of the \NSGA can be found in Algorithms~\ref{alg:NSGAD}, where we note that the presentation here is optimized for a uniform treatment of this \NSGA and the variant with sequential selection to be discussed now. 

Noting that the removal of an individual changes the crowding distance of the remaining individuals, Kukkonen and Deb~\cite{KukkonenD06} proposed to take into account this change, that is, to \emph{sequentially remove individuals} with smallest crowding distance and update the crowding distance of the remaining individuals. This was shown to give superior results in the empirical study~\cite{KukkonenD06} and the mathematical analysis~\cite{ZhengD24approx}. 

\begin{algorithm}[!ht]
    \caption{The \NSGA algorithm in its classic version~\cite{DebPAM02} and with sequential selection~\cite{KukkonenD06}. When replacing in either version the original crowding distance with the truthful crowding distance proposed in this work, we obtain our truthful (sequential) \NSGAD.}
    \begin{algorithmic}[1]
    \STATE {Uniformly at random generate the initial population $P_0=\{x_1,x_2,\dots,x_N\}$ with $x_i\in\{0,1\}^n,i=1,2,\dots,N.$}\label{ste:initialize}
    \FOR{$t = 0, 1, 2, \dots$} \label{ste:iterate}
    \STATE {Generate the offspring population $Q_t$ with size $N$}\label{ste:generate}
    \STATE {Use non-dominated sorting to divide $R_t=P_t\cup Q_t$ into $F_1,F_2,\dots$}
    \label{ste:sort}
    \STATE {Find $i^* \ge 1$ such that $\sum_{i=1}^{i^*-1}|F_i| < N$ and $\sum_{i=1}^{i^*}|F_i| \ge N$. Let $N_r:=\sum_{i=1}^{i^*}|F_i| - N$}\label{ste:rank}
    \STATE{Compute the crowding distance of each individual in $F_{i^*}$}\label{ste:inid}
    \FOR{$d=1,\dots,N_r$}\label{stp:start}
    \STATE{Remove one individual in $F_{i^*}$ with the smallest crowding distance, chosen at random in case of a tie}
    \STATE{For sequential selection only: Update the crowding distances of the individuals affected by the removal}
    \ENDFOR\label{stp:end}
    \STATE {$P_{t+1}:=\bigcup_{i=1}^{i^*}F_i$}\label{ste:new parents}
    \ENDFOR 
    \end{algorithmic}
    \label{alg:NSGAD}
\end{algorithm}

\section{Classic and Truthful Crowding Distance}\label{sec:dCD}

In this section, we first describe the original crowding distance used in the \NSGA of Deb et al.~\cite{DebPAM02} and compare it with other ways to select a subset of individuals from the critical front of a non-dominated sorting (secondary selection criterion). This comparison motivates the development of a modification of the crowding distance, called \emph{truthful crowding distance}, done in the second half of this section.

\subsection{Original Crowding Distance}

When selecting the next population, the \NSGA, {\NSGAthree}, and \SMS first perform non-dominated sorting, resulting in a partition $F_1, F_2, \dots$ of the combined parent and offspring population into fronts of pair-wise non-dominated individuals. For a suitable number $i^*$, the first $i^*-1$ fronts are all taken into the next population; from $F_{i^*}$ a subset is selecting according to a secondary criterion. 

For the \NSGA, this secondary criterion is the crowding distance. The crowding distance of an individual $x$ in a set $S$ is the sum, over all objectives, of the normalized distances of the two neighboring objective values. 
Formally, let $m$ be the number of objectives and let $S=\{S_1,\dots,S_{|S|}\}$ be the set of individuals. For each $i \in [1..m]$, let $S_{i,1},\dots,S_{i,|S|}$ be the sorted list of $S$ w.r.t.~$f_i$. How ties in this sortings are broken has to be specified by the algorithm designer, we do not take any particular assumptions on that issue. For  $x \in S$, we denote by $i_x$ its position in the sorted list w.r.t.~$f_i$, that is, $x=S_{i,i_x}$. The crowding distance of $x$ then is
\begin{equation}
\cDis(x)=\begin{cases}
+\infty, &\text{if $i_x\in\{1,|S|\}$ for some $i\in[1..m]$,}\\
\sum_{i=1}^m&\frac{|f_i(S_{i,i_x-1})-f_i(S_{i,i_x+1})|}{|f_i(S_{i,1})-f_i(S_{i,|S|})|}, \text{otherwise.}
\end{cases}
\label{eq:cDis}
\end{equation}

The simple and intuitive definition of the crowding distance puts it ahead of other secondary criteria in several respects
. Compared to the hypervolume contribution used by the \SMS, the crowding distance can be computed very efficiently, namely in time $O(m |S| \log |S|)$, which from $m \ge 4$ on is significantly faster than the best known approach to compute the hypervolume contribution, an algorithm with runtime  $O(|S|^{m/3} \polylog |S|)$ from the break-through paper~\cite{Chan13}, see also the surveys~\cite{ShangIHP20,GuerreiroPF21}.

Compared to the reference point mechanism employed by the \NSGAthree, the crowding distance needs no parameters to be set. In contrast, the \NSGAthree requires a normalization procedure (for which several proposals exist) and the set of reference points (for which several constructions exists, all having at least the number of reference points as parameter).

Besides many successful applications in practice, also a decent number of mathematical results show that the \NSGA with its crowding distance secondary selection criterion is able to compute or approximate the Pareto front of various classic problems~\cite{ZhengLD22,BianQ22,DoerrQ23tec,DoerrQ23crossover,DangOSS23aaai,DangOSS23gecco,CerfDHKW23,ZhengLDD24,ZhengD24approx}. 

However, the positive results for the \NSGA are limited to bi-objective problems, and this limitation is intimately connected to the crowding distance. As demonstrated by Zheng and Doerr~\cite{ZhengD23many}, the \NSGA fails to compute the Pareto front of the simple \omm problem once the number of objective is at least three. The reason deduced in that work is the independent treatment of the objectives in calculating the crowding distance. Subsequent positive results for the \NSGAT~\cite{WiethegerD23,OprisDNS24} and \SMS~\cite{ZhengD24,WiethegerD24}) for three or more objectives support the view that the crowding distance has intrinsic short-comings.


\subsection{Truthful Crowding Distance}\label{subsec:dCD}

Given the undeniable algorithmic advantages of the crowding distance and its high acceptance by practitioners, we now design a simple and efficient variant of crowding distance that also works well for many objectives.

As pointed out in the example in~\cite{ZhengD23many}, the original crowding distance allows that points far away from a solution still cause it to have a small crowding distance. This counter-intuitive and undesired behavior stems from the fully independent consideration of the objectives in the calculation of the crowding distance: In~(\ref{eq:cDis}), the $i$-th summand only relies on distances w.r.t.~$f_i$ and ignores possibly large distances stemming from other objectives $f_j, j\ne i$. 

To avoid the undesired influence of points far away on the crowding distance components, but at the same time allow for a highly efficient computation of the crowding distance, we proceed as follows. (i)~We replace the (normalized) distance in the $i$-th objective by the (normalized) $L_1$ distance. This avoids that points far in the objective space lead to low crowding distance values. (ii)~We keep the property that the crowding distance is the sum of the crowding distance contributions of the different objectives. This was the key reason why the original crowding distance can be computed very efficiently. (iii)~In the computation of the $i$-th crowding distance contribution, we also keep working with the individuals sorted in order of descending $f_i$ value. (iv)~Noting that the use of the $L_1$ distance might imply that the point closest to some $S_{i.j}$ is not necessarily $S_{i.j-1}$ or $S_{i.j+1}$, we consider the minimum $L_1$ distance among $S_{i.k}, k < j$. We note that this renders our crowding distance less symmetric than the original crowding distance, but we could not see a reason to let, in the language of the original crowding distance, $|f_i(S_{i.j}) - f_i(S_{i.j-1})|$ contribute to both the crowding distance of $S_{i.j}$ and $S_{i.j-1}$. In fact, we shall observe that this slightly less symmetric formulation will reduce the number of solutions with identical objective vector and positive crowding distance. (v)~Finally, we shall assume that the different sortings used sort individuals with identical objective vectors in the same order (\emph{correlated tie-breaking}). The original crowding distance does not specify how to break such ties, but any stable sorting algorithm will have this property, so this assumption is not very innovative. As observed in~\cite{BianQ22}, this assumption of correlated tie-breaking can reduce the minimum required population size for certain guarantees to hold. 

We now give the \textbf{formal definition} of our crowding distance, which we call \emph{truthful crowding distance} to reflect that fact that it better describes how isolated a solution is. Let $S=\{S_1,\dots,S_{|S|}\}$ be a set of pair-wise non-dominated individuals. For all $i \in [1..m]$, let $S_{i.1},\dots,S_{i.{|S|}}$ be a sorted list of $S$ in descending order of $f_i$. Assume correlated tie-breaking, that is, if two individuals have identical objective values, then they appear in all sortings in the same order. 

If an individual $x$ appears as the first element of some sorting, that is, $x = S_{i.1}$ for some $i \in [1..m]$, then its truthful crowding distance $\tcd(x)$ is $\tcd(x) := \infty$. Otherwise, its crowding distance shall be the sum $\tcd(x) = \sum_{i=1}^m \tcd_i(x)$ of the crowding distance contributions $\tcd_i(x)$, which we define now.

To this aim, let $i \in [1..m]$ and $j \in [2..|S|]$ such that $x = S_{i.j}$. For $k < j$, we define the normalized $L_1$ distance by 
\begin{align*}
d(S_{i.k},S_{i.j}){}&{}:=
\sum_{a=1}^m \frac{|f_a(S_{i.k})-f_a(S_{i.j})|}{f_a(S_{a.1})-f_a(S_{a.{|S|}})},
\end{align*}
where we count summands ``$0/0$'' as zero (this happens in the exotic case that in some objective, only a single objective value is present in $S$).
With this distance, the $i$-th crowding distance contribution $\tcd_i(x)$ is defined as the smallest distance between $S_{i.j}$ and a solution 
in an earlier position in the $i$-th list:
\[
\tcd_i(x) := \min_{k < j} d(S_{i.k},S_{i.j}). 
\]
This defines our variant of the crowding distance, called truthful crowding distance. The pseudocode of an algorithm computing it is given in Algorithm~\ref{alg:dcDis}. As is easy to see, this algorithm has a time complexity quadratic in the size of the set $S$, more precisely, $\Theta(m|S|^2)$. This is more costly than the computation of the original crowding distance in time $\Theta(m |S| \log |S|)$. Since the best known time complexity of non-dominated sorting in the general case is $\Theta(m|S|^2)$ and no better runtime can be expected in the general case~\cite{YingchareonthawornchaiRLTD20}, this moderate increase in the complexity of computing the crowding distance appears tolerable.

 
\begin{algorithm}[tb]
    \caption{Computation of the truthful crowding distance $\dCD(S)$}
    \textbf{Input:} $S=\{S_1,\dots,S_{|S|}\}$, a set of individuals\\
    \textbf{Output:} $\dCD(S)=(\dCD(S_1),\dots,\dCD(S_{|S|}))$, where $\dCD\left(S_i\right)$ is the truthful crowding distance for $S_i$
		
    \begin{algorithmic}[1]
    \STATE $\dCD(S):=(0,\dots,0)$
    \FOR {each objective $f_i, i=1,\dots, m$}
    \STATE {Sort $S$ in order of descending $f_i$ value with correlated tie-breaking: $S_{i.1},\dots,S_{i.{|S|}}$}
    \ENDFOR
    \FOR {each objective $f_i, i=1,\dots, m$}
        \STATE {$\dCD\left(S_{i.1}\right):=+\infty$}
    \FOR {$j=2,\dots, |S|$}
    \FOR {$k=1,\dots, j-1$}
    \STATE {$d(S_{i.k},S_{i.j}):=\sum_{a=1}^m\frac{|f_a(S_{i.k}) - f_a(S_{i.j})|}{f_a(S_{a.1})-f_a(S_{a.{|S|}})}$}
    \ENDFOR
    \STATE {$\dCD(S_{i.j}):=\dCD(S_{i.j}) + \min\limits_{k=1,\dots, j-1} d(S_{i.k},S_{i.j})$}
    \ENDFOR
    \ENDFOR
    \end{algorithmic}
    \label{alg:dcDis}
\end{algorithm}

As said, we propose in this work to use the classic \NSGA or the sequential \NSGA, but with the original crowding distance replaced by the truthful crowding distance. We call the resulting algorithms \emph{truthful (sequential) \NSGA}, abbreviated (sequential) \NSGAD. 

\section{Runtime Analysis: Computing the Pareto Front}\label{sec:many}

Having introduced the truthful crowding distance and the truthful (sequential) \NSGA, denoted by \NSGAD, in this and the subsequent section we will conduct several runtime analyses of this algorithm. 
The results in this section will in particular show that the \NSGAD can efficiently optimize the many-objective problems, in contrast to the exponential runtime of the original \NSGA on \momm~\cite{ZhengD23many}.


\subsection{Not Losing Pareto Front Points}\label{sssec:dCDadv}

The key ingredient to all proofs in this section is what we show in this subsection (in Theorem~\ref{thm:pfkept}), namely that the \NSGAD with sufficiently large population size cannot lose Pareto optimal solution values (and more generally, can lose solution values only by replacing them with better ones). This is a critical difference to the classic \NSGA, as shown in~\cite{ZhengD24}.

A step towards proving this important property is the following lemma, which asserts that for each objective vector of the population exactly one individual with this function value has a positive crowding distance.

\begin{lemma}
Let $m\in\N$ be the number of objectives of the discussed function $f=(f_1,\dots,f_m)$. Let $S$ be a population of individuals in $\{0,1\}^n$. Assume that we compute the truthful crowding distance $\dCD(S)$ via Algorithm~\ref{alg:dcDis}. Then for any function value $v\in f(S)$, exactly one individual $x \in S$ with $f(x)=v$ has a positive truthful crowding distance (and the others have a truthful crowding distance of zero). 
\label{lem:dCDp}
\end{lemma}

\begin{proof}
Let $S = \{S_1,\dots, S_{|S|}\}$. Let $I=\{j\in[1..|S|] \mid f(S_j)=v\}$ be the index set of the individuals in $S$ with function value $v$. For the sorted list $S_{i.1},\dots,S_{i.|S|}$, let $I^{(i)}_1,\dots,I^{(i)}_{|I|}$ be the increasing sequence of the indices $j$ of all elements in $I$, that is, we have $I = \{S_{i.I^{(i)}_k} \mid k \in [1..|I|]\}$ and $I^{(i)}_k < I^{(i)}_{k+1}$ for all $k \in [1..|I|-1]$. By definition of correlated tie-breaking, we know that for all $k=1,\dots,|I|$, we have $S_{1.I^{(1)}_k}= \dots =S_{m.I^{(m)}_k}$. 

We first show that $S_{1.I^{(1)}_1}$ has a positive crowding distance. If $I^{(1)}_1=1$, then $\dCD\big(S_{1.I^{(1)}_1}\big)=+\infty>0$. If $I^{(1)}_1 > 1$, then all individuals $S_{1.j}$ with $j<I^{(1)}_1$ have function values different from~$v$. That is, for each $j <I^{(1)}_1$ there exists $i'\in [1..m]$ such that $f_{i'}\left(S_{1.j}\right) \neq f_{i'}\big(S_{1.I^{(1)}_1}\big)$. Recalling that we regard a normalized version of the $L_1$ distance, this implies   $d\big(S_{1.j},S_{1.I^{(1)}_1}\big)>0$ for all $j<I^{(1)}_1$, thus $0 < \tcd_1\big(S_{1.I^{(1)}_1}\big) \le \tcd\big(S_{1.I^{(1)}_1}\big)$ as desired.

We end the proof by showing that for $k > 1$, we have $\tcd(S_{1.I^{(1)}_k}) = 0$. To this aim, we observe that for all $i \in [1..m]$, we have $S_{i.I_k^{(i)}} = S_{1.I^{(1)}_k}$, this individual and $S_{i.I_1^{(i)}}$ have the same $f$-value giving $d\big(S_{i.I^{(i)}_1},S_{i.I^{(i)}_k}\big)=0$, and $I^{(i)}_1 < I^{(i)}_k$. Consequently, $\dCD_i\big(S_{i.I^{(i)}_k}\big)$ is zero. Since this holds for all $i=1,\dots, m$,  we have $\dCD\big(S_{1.I^{(1)}_k}\big)=0$.
\end{proof}


From Lemma~\ref{lem:dCDp}, we derive our main technical tool asserting that a sufficiently high population size ensures that the \NSGAD does not lose desirable solutions. 

\begin{theorem}
Let $m\in\N$ be the number of objectives for a given $f=(f_1,\dots,f_m)$, and let $\overline{M} \in \N$ be such that any set $S$ of incomparable solutions satisfies $|S| \le \overline{M}$. 
Consider using the (sequential) \NSGAD with population size $N\ge \overline{M}$ to optimize $f$. For any solution $x$ in the combined parent and offspring population $R_t$, in the next and all future generations, there is at least one individual $y$ such that $y\succeq x$.
\label{thm:pfkept}
\end{theorem}

We note that for many problems, the maximum size of a set of incomparable solutions $\overline M$ is already witnessed by the Pareto front $F^*$ (that is, $|F^*| = \overline M$). Hence the requirement $N \ge \overline M$ of the theorem is needed anyway to ensure that the algorithm can store a population $P$ with $f(P) = F^*$.

\begin{proof}[Proof of Theorem~\ref{thm:pfkept}]
We conduct the proof for the more complicated case of the sequential \NSGAD, the proof for the \NSGAD follows from a subset of the arguments.

Let $x\in R_t$ be in the $i$-th front, that is, $x\in F_i$. If $i<i^*$, from the selection in the \NSGAD, we know that $x$ will enter into the next generation, and $y=x$ suffices. If $i\ge i^*$ and $i^*> 1$, then there exists a solution $y\in F_{1}$ such $y\succeq x$ and $y$ will enter into the next generation. Hence, we only need to discuss the case $i=i^*=1$
in the following. 

From Lemma~\ref{lem:dCDp}, we know that for each function value in $f(F_{1})$ there is an individual in $R_t$ with positive truthful crowding distance. Let $y\in R_t$ be the individual with $f(y)=f(x)$ and with positive truthful crowding distance. Then $y\in F_{1}$ as well and $y\succeq x$. 

From the definition of $\overline{M}$ and Lemma~\ref{lem:dCDp} again (now referring to the assertion that there is at most one individual per objective value with positive truthful crowding distance), we know that before each the removal, there are at most $\overline{M}$ individuals in $F_{1}$ with positive truthful crowding distance. 
Since $N \ge \overline M$ individuals of $F_1$ survive, this means that in the whole removal procedure, only individuals with zero truthful crowding distance will be removed. 
By the definition of the truthful crowding distance, a crowding distance of zero means that there is a second individual with same objective value appearing in all sortings before the first. 
Hence the removal of the first individual does not change the truthful crowding distance of any other individual (by definition of the truthful crowding distance). Hence, all individuals having initially a positive truthful crowding distance, including~$y$, will survive to the next population.
\end{proof}

\begin{corollary}
  Under the assumptions of Theorem~\ref{thm:pfkept}, once a solution with a given Pareto optimal solution value is found, such a solution will be contained in the population for all future generations.
\end{corollary}

\subsection{Runtime Results for Many Objectives}

We now build on the structural insights on the \NSGAD gained in the previous subsection and show that this algorithm can easily optimize the standard benchmarks, roughly with the same efficiency as the global SEMO algorithm, a minimalistic MOEA mostly used in theoretical analyses, and the two classic MOEAs \NSGAthree and \SMS. This in particular shows that the truthful \NSGA does not face the problems the classic \NSGA faces when the number of objectives is three or more.

For our analysis, we are lucky that we can heavily build on the work~\cite{WiethegerD24}, in which near-tight runtime guarantees for many-objective optimization were proven. As discussed in \cite[Section~5]{WiethegerD24}, their proofs only rely on two crucial properties: (i)~that solutions values are never lost except when replaced by better ones, and (ii)~that there is a number $S$ such that for any individual $x$ in the population, with probability $1/S$ this $x$ is selected as parent and an offspring is generated from it via bit-wise mutation with mutation rate~$\frac 1n$. 

It is easy to see that these properties are fulfilled for our (sequential) \NSGAD when using bit-wise mutation. Property~(i) is just the assertion of Theorem~\ref{thm:pfkept}. Property~(ii) follows immediately from the definition of our algorithm: the probability for this event is $1$ for fair selection and $\frac{1}{1-p}(1-1/N)^N = \Theta(1)$ for random selection and crossover rate~$p<1$. 
With these considerations, we immediately extend the results of~\cite{WiethegerD24} to the truthful (sequential) \NSGA.

\begin{theorem}
    Consider using the (sequential) \NSGAD with problem size $N\ge M:=(2n/m+1)^{m/2}$, fair or random selection, standard bit-wise mutation with mutation rate $1/n$, and possibly crossover with rate less than one in the case of random selection, to optimize the \momm or \mcocz benchmarks. Then in an expected number of $O(nm)+O(m^2 \ln n)$ iterations, the full Pareto front of the \momm or \mcocz benchmarks is covered by the population. 
    \label{thm:momm}
\end{theorem}

\begin{theorem}
    Consider using the (sequential) \NSGAD with problem size $N\ge M:=(2n/m+1)^{m/2}$, fair or random selection,  standard bit-wise mutation with mutation rate $1/n$, and possibly crossover with rate less than one in the case of random selection, to optimize the \mlotz benchmark. Then in an expected number of $O(n^2/m)+O(mn \ln (n/m))+ O(n\ln n)$ iterations, the full Pareto front of the \mlotz benchmark is covered.  
\end{theorem}

\begin{theorem}
    Let $k \in [2..n/m]$. Let $M=(2n/m-2k+3)^{m/2}$. Consider using the (sequential) \NSGAD with problem size $N\ge \overline{M}$, fair or random selection,  standard bit-wise mutation  with mutation rate $1/n$, and possibly crossover with rate less than one in the case of random selection, to optimize $\mojzj_k$. Then in an expected number of $O(mn^k)$ iterations, the full Pareto front of the $\mojzj_k$ benchmark is covered.
\end{theorem}

We have not defined the benchmark problems regarded in the above results, both because they are the most common benchmarks in the theory of MOEAs and because our proof do not directly refer to them (all problem-specific arguments are taken from~\cite{WiethegerD24}). The reader interested in the definitions can find them all in~\cite{WiethegerD24}.


\subsection{Runtime Results for Two Objectives}

We now turn to the bi-objective versions of the benchmarks studied above and the \DLTB benchmark. Here the classic \NSGA was shown to be efficient in previous work~\cite{ZhengLD22,ZhengD23aij,BianQ22,DoerrQ23tec,ZhengLDD24}. Using the same arguments as in the previous section, we show the following results for the (sequential) \NSGAD. In the runtimes, they agree with the known asymptotic results for the classic \NSGA. However, the minimum required population size (which has a direct influence on the cost of one iteration) is by a factor of two or four smaller than in the previous works. Since it is equal to the size of the Pareto front, it is clear that even smaller population sizes cannot be employed.

\begin{theorem}
    Consider using the (sequential) \NSGAD with problem size $N\ge n+1$, fair or random selection, standard bit-wise mutation with mutation rate $1/n$, and possibly crossover with rate less than one in the case of random selection, to optimize \omm or \cocz. Then in an expected number of $O(n\log n)$ iterations, the full Pareto front of  \omm or \cocz is covered.
\end{theorem}
\begin{theorem}
    Consider using the (sequential) \NSGAD with problem size $N\ge n+1$, fair or random selection, standard bit-wise mutation with mutation rate $1/n$, and possibly crossover with rate less than one in the case of random selection, to optimize \lotz. Then in an expected number of $O(n^2)$ iterations, the full Pareto front of  \lotz is covered.
\end{theorem}
\begin{theorem}
    Let $k\in[1..n/2]$. Consider using the (sequential)  \NSGAD with problem size $N\ge n-2k+3$, fair or random selection, standard bit-wise mutation with mutation rate $1/n$, and possibly crossover with rate less than one in the case of random selection, to optimize $\ojzj_k$. Then in an expected number of $O(n^k)$ iterations, the full Pareto front of $\ojzj_k$ is covered.
\end{theorem}
\begin{theorem}
    Consider using the (sequential) \NSGAD with problem size $N\ge n+1$, fair or random selection, standard bit-wise mutation with mutation rate $1/n$, and possibly crossover with rate less than one in the case of random selection, to optimize \DLTB. Then in expected $O(n^3)$ iterations, the full Pareto front of the \DLTB is covered.
\end{theorem}



\section{Approximation Ability and Runtime}\label{sec:appr}

In Section~\ref{sec:many}, we proved that the standard and sequential \NSGAD can efficiently optimize the many-objective \mojzj, \momm, \mcocz, and \mlotz benchmarks as well as the popular bi-objective \ojzj, \omm, \cocz, \lotz, and \DLTB benchmarks. In this section, we will consider the approximation ability when the population size is too small to cover the full Pareto front. We will prove that the sequential \NSGAD has a slightly better approximation performance than the sequential \NSGA and the steady-state \NSGA for \omm~\cite{ZhengD24approx}. 

We note that there are no proven approximation guarantees for non-sequential variants of the \NSGA (except for the steady-state version) so far and the mathematical results in~\cite{ZhengD24approx} suggest that such results might be difficult to obtain. For that reason, we do not aim at such results for the truthful \NSGA. We also note that so far there is no theoretical study on the approximation ability of the \NSGA other than for the (bi-objective) \omm benchmark~\cite{ZhengD24approx}. We shall therefore also only consider this problem. We expect that results for larger numbers of objectives or other benchmarks need considerably new methods as already the approximation measure \mei is may not be suitable then. 

The following lemma gives a useful criterion for individuals surviving into the next generation.
\begin{lemma}
    Consider using the sequential \NSGAD with population size $N\ge 2$ to optimize \omm with problem size $n$. Assume that the two extreme points $0^n$ and $1^n$ are in the population $P_{t_0}$. Then for any generation $t\ge t_0$, in Steps~\ref{stp:start} to~\ref{stp:end} of Algorithm~\ref{alg:NSGAD}, any individual with truthful crowding distance more than $\frac{4}{N-1}$ (including two extreme points) will survive to~$P_{t+1}$. 
    \label{lem:surdCD}
\end{lemma}
\begin{proof}
  Consider some iteration $t \ge t_0$. Let $R$ denote the combined parent and offspring population. We recall that $P_{t+1}$ is constructed from $R$ by sequentially removing individuals with the smallest current $\dCD$-value. By definition, the removal of an individual will not decrease the truthful crowding distance of the
  remaining individuals.
  In particular, individuals that initially have an infinite truthful crowding distance or have a crowding distance of at least $\frac{4}{N-1}$ will keep this property throughout this iteration. 
  
  It is not difficult to see that there are exactly one copy of $0^n$ and of $1^n$ with infinite truthful crowding distance. Since $N\ge 2$, both individuals will be kept to the next and all future generations. 
  
  Now consider $R$ at some stage of the sequential selection process towards $P_{t+1}$, that is, with some individuals already removed. Let $r := |R|$ and let $s_1^1,\dots,s_r^1$ and $s_1^2,\dots,s_r^2$ be the two lists representing $R$ w.r.t. decreasing values of $f_1$ and $f_2$, respectively. Let $j_1\in[1..r]$ be the position of $s_1^1$ in the sorted list w.r.t. $f_2$, that is, $s^2_{j_1} = s^1_1$. Likewise, let $i_1\in[1..r]$ be the position of $s_1^2$ in the sorted list w.r.t. $f_1$, that is, $s^1_{i_1} = s^2_1$. For any $x\in R$, we have unique $i,j\in[1..r]$ such that $x=s_i^1=s_j^2$. 
    Since $0^n$ and $1^n$ are in $R$, then $(s_1^1,s_r^1)=(1^n,0^n)$ and $(s_1^2,s_r^2)=(0^n,1^n)$, and for $x$ with $i\in[2..r]\setminus\{i_1\}$ and $j\in [2..r]\setminus \{j_1\}$, we have
    \begin{align*}
        \dCD(x)={}&{}\left(\frac{f_1(s_{i-1}^1)-f_1(s_{i}^1)}{n}+\frac{f_2(s_{i}^1)-f_2(s_{i-1}^1)}{n}\right)\\
        &{}+\left(\frac{f_1(s_j^2)-f_1(s_{j-1}^2)}{n}+\frac{f_2(s_{j-1}^2)-f_2(s_j^2)}{n}\right).
    \end{align*}
    Noting that $s_1^1 \ne s_1^2$ since $s_1^1=1^n, s_1^2=0^n$, we
    compute
    \begin{align*}
    \sum_{x\in R\setminus\{s_1^1,s_1^2\}}{}&{}\dCD(x)=  \sum_{i\in[2..r]\setminus\{i_1\}} \left(\frac{f_1(s_{i-1}^1)-f_1(s_{i}^1)}{n}+\frac{f_2(s_{i}^1)-f_2(s_{i-1}^1)}{n}\right)\\
    &{}+ \sum_{j\in[2..r]\setminus\{j_1\}} \left(\frac{f_1(s_j^2)-f_1(s_{j-1}^2)}{n}+\frac{f_2(s_{j-1}^2)-f_2(s_j^2)}{n}\right)\\
    ={}&{}\left(\sum_{i=2}^{i_1-1}+\sum_{i=i_1+1}^{r}\right) \left(\frac{f_1(s_{i-1}^1)-f_1(s_{i}^1)}{n}+\frac{f_2(s_{i}^1)-f_2(s_{i-1}^1)}{n}\right)\\
    &{}+ \left(\sum_{j=2}^{j_1-1} + \sum_{j=j_1+1}^{r}\right) \left(\frac{f_1(s_j^2)-f_1(s_{j-1}^2)}{n}+\frac{f_2(s_{j-1}^2)-f_2(s_j^2)}{n}\right)\\
    ={}&{}\frac{(f_1(s_1^1)-f_1(s_{i_1-1}^1))+(f_2(s_{i_1-1}^1)-f_2(s_{1}^1))}{n}\\
    &{}+\frac{(f_1(s_{i_1}^1)-f_1(s_{r}^1))+(f_2(s_{r}^1)-f_2(s_{i_1}^1))}{n}\\
    &{}+\frac{(f_1(s_{j_1-1}^2)-f_1(s_{1}^2))+(f_2(s_{1}^2)-f_2(s_{j_1-1}^2))}{n}\\
    &{}+\frac{(f_1(s_r^2)-f_1(s_{j_1}^2))+(f_2(s_{j_1}^2)-f_2(s_{r}^2))}{n}\\
    ={}&{}\frac{f_1(s_1^1)-f_1(s_{r}^1)+f_1(s_{i_1}^1)-f_1(s_{i_1-1}^1)}{n}\\
    &{}+\frac{f_2(s_{r}^1)-f_2(s_{1}^1)+f_2(s_{i_1-1}^1)-f_2(s_{i_1}^1)}{n}\\
    &{}+\frac{f_1(s_r^2)-f_1(s_{1}^2)+f_1(s_{j_1-1}^2)-f_1(s_{j_1}^2)}{n}\\
    &{}+\frac{f_2(s_{1}^2)-f_2(s_{r}^2)+f_2(s_{j_1}^2)-f_2(s_{j_1-1}^2)}{n}\\
    \le{}&{} \frac{f_1(s_1^1)-f_1(s_{r}^1)+f_2(s_{r}^1)-f_2(s_{1}^1)}{n}\\
    &{}+\frac{f_1(s_r^2)-f_1(s_{1}^2)+f_2(s_{1}^2)-f_2(s_{r}^2)}{n}\\
    ={}&{} 4,
    \end{align*}
    where the last inequality uses $f_1(s_{i_1-1}^1) \ge f_1(s_{i_1}^1)$ and $f_2(s_{j_1-1}^2)\ge f_2(s_{j_1}^2)$ due to the sorted lists, and further $f_2(s_{i_1-1}^1) \le f_2(s_{i_1}^1)$ and $f_1(s_{j_1-1}^2)\le f_1(s_{j_1}^2)$ since in a bi-objective incomparable set, any sorting 
    with respect to the first objective is a sorting in inverse order for the second objective. 
     Since $|R\setminus\{s_1^1,s_1^2\}|\ge N-1$, we know that at least one of individuals in $R\setminus\{s_1^1,s_1^2\}$ will have its $\dCD\le \frac{4}{N-1}$, and thus any individual with its $\dCD>\frac{4}{N-1}$ will not be removed.
\end{proof}

The following lemma shows that once the two extreme points are in the population, a linear runtime suffices to obtain a good approximation of the Pareto front of \omm.

\begin{lemma}\label{lem:approx}
    Consider using the sequential \NSGAD with population size $N\ge 2$, fair or random parent selection, one-bit mutation or standard bit-wise mutation, crossover with constant rate below $1$ or no crossover, to optimize \omm with problem size $n$. Let $L:=\max\{\frac{2n}{N-1},1\}$. Assume that the two extreme points $0^n$ and $1^n$ are in the population for the first time at some generation $t_0$. Then after $O(n)$ more iterations (both in expectation and with high probability), the population will have its an \mei value of at most $L$. It remains in this state for all future generations. 
\end{lemma}

\begin{proof}[Proof sketch]
    We use similar argument as in the proof of~\cite[Lemma~14]{ZhengD24approx} and only show the difference here. Let $i\in[0..n-1], t\ge t_0$. Let $X_t$ and $Y_t$ be the lengths of the empty intervals containing $i+0.5$ in $f(P_t)$ and in $f(R_t)$, respectively. We first show that if $Y_t\le M$ for some $M\ge L$, then $X_{\tau} \le M$ for all $\tau >t$.
    If not, since $Y_t\le X_t$, w.l.o.g., we assume that $X_{t+1}>M$. Let $x$ be the individual whose removal lets the length of the empty interval containing $i+0.5$ increase from a value of at most $M$ to a value larger than~$M$. Then before the removal, $f_1(x)$ must be one of the end points of the empty interval containing $i+0.5$, w.l.o.g, the left end point (the smaller $f_1$ value). 
    We also know that the empty interval containing $i+0.5$ after the removal of $x$ has lengths equal to the truthful crowding distance of $x$ multiplied by $n/2$. Hence, we know that
    \begin{align*}
        \dCD(x) > \frac{2M}{n} \ge \frac{2L}{n}= \frac{4}{N-1},
    \end{align*}
    which is contradicts our insight from Lemma~\ref{lem:surdCD} that an $x$ with $\dCD>4/(N-1)$ cannot be removed.

    The remaining argument about the first time the population has $\mei \le L$ is exactly the same as in the proof of~\cite[Lemma~14]{ZhengD24approx}, except for the case with crossover. Since crossover is used with a constant probability less than one, there is a constant rate of iterations using mutation only. 
    The arguments analyzing the selection are independent of the variation operator (so in particular, the empty interval lengths are non-increasing when at least~$L$). Consequently, by simply ignoring a possible profit from crossover iterations, we obtain the same runtime guarantee as in the mutation-only case.
\end{proof}

Noting that the maximal function value of $f_1(P_t)$ and $f_2(P_t)$ cannot decrease (there always is one individual witnessing this value and having infinite truthful crowding distance), we easily obtain that in expected $O(n\log n)$ iterations both extreme points $0^n$ and $1^n$ are reached for the first time, and also for all future iterations. This can be shown with a proof analogous to the one of~\cite[Lemma~15]{ZhengD24approx}. Therefore, we have the following main result on the approximation ability and runtime of the sequential \NSGAD.
\begin{theorem}
    Consider using the sequential \NSGAD with population size $N\ge 2$, fair or random parent selection, one-bit mutation or standard bit-wise mutation, crossover with constant rate below $1$ or no crossover, to optimize \omm with problem size $n$. Then after an expected number of $O(n\log n)$ iterations, the population contains $0^n$ and $1^n$ and satisfies $\mei\le \max\{\frac{2n}{N-1},1\}$. After that, these conditions will be kept for all future iterations.
\label{thm:nsgadapp}
\end{theorem}

Note that the best possible value for the $\mei$ is $\lceil \frac{n}{N-1}\rceil$~\cite{ZhengD24approx}. Theorem~\ref{thm:nsgadapp} shows that the sequential \NSGAD can reach a good approximation guarantee of $\mei$ of the optimal value multiplied by at most only a factor of two. We also note that for the \NSGA with sequential survival selection using the classic crowding distance, an approximation guarantee of $\mei\le\max\{\frac{2n}{N-3},1\}$ (also within $O(n\log n)$ iterations) was shown in~\cite{ZhengD24approx}. The slightly better approximation ability shown above (within the same runtime), as the proof of Lemma~\ref{lem:approx} shows, stems from the fact that our definition of the crowding distance admits at most one individual with infinite crowding distance contribution per objective, whereas the classic crowding distance admits two. 

\section{Experiments}

To complement our theoretical findings, we now show a few experimental results. These meant as illustration of our main (mathematical) results, not as substantial stand-alone results.

\textbf{Computing the full Pareto front, many objectives:} Our main theoretical result was a proof that the \NSGAD can efficiently solve many-objective problems, different from the classic \NSGA, where an exponential lower bound was shown for the \omm problem. To illustrate how the \NSGAD solves many-objective problems, we regard the 4-objective \omm problem. That the \NSGA cannot solve this problem efficiently was shown, also experimentally, in~\cite{ZhengD23many}. We hence study only the (sequential) \NSGAD as algorithm, using random selection, bit-wise mutation, no crossover, and the minimal possible population size $N = M$ (the Pareto front size) and $N = 2M$. We also use the GSEMO toy algorithm. In Figure~\ref{fig:4omm}, we display the median (over 20 runs) number of function evaluations these algorithms took to cover the full Pareto front of $4$-\omm for different problem sizes~$n$. 

\begin{figure}
\centering
\includegraphics[width=0.81\columnwidth]{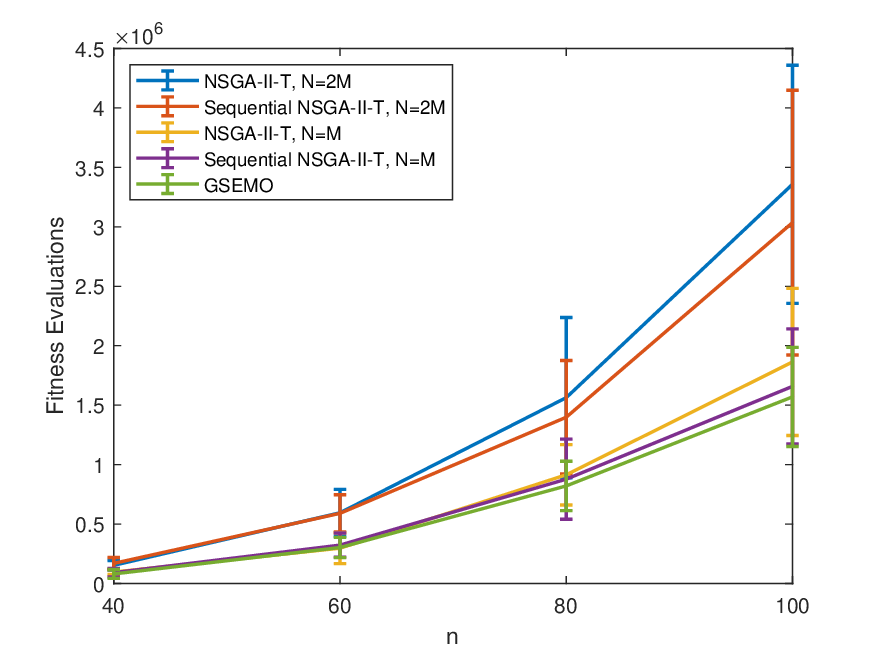}
\caption{Median number (with $1$st and $3$rd quartiles, in 20 runs) of function evaluations to compute the full Pareto front of the $4$-objective \omm problem.}
\label{fig:4omm}
\end{figure}

We observe that all algorithms efficiently find the full Pareto front, in drastic contrast to the results for the classic \NSGA in~\cite[Figure~1 and 2]{ZhengD23many}. Not surprisingly, a larger population size is not helpful, which shows that it is good that the \NSGAD admits smaller population sizes than the classic \NSGA. Also not surprisingly, the sequential versions give slightly better results. The minimal inferiority of the (sequential) \NSGAD (with $N=M$) to the toy GSEMO does not mean a lot given that this algorithm is rarely used in practice.

\textbf{Computing the full Pareto front, two objectives:} We conducted analoguous experiments for two objectives, where a comparison with the \NSGA is interesting. The proven guarantees for the \NSGA require a population size of $N \ge 4M$, where this algorithm is clearly slower than all others regarded by us. We therefore did some preliminary experiments showing that already for $N = 1.5M$ the \NSGA consistently is able to solve our problem instances. The results for this \NSGA, the \NSGAD with optimal population size $N = M$ and with $N = 1.5M$, and the GSEMO are shown in Figure~\ref{fig:ommPF}. With this optimized population size for the \NSGA, all algorithms show a roughly similar performance on the $2$-objective \omm problem, with the \NSGA slightly ahead.

\begin{figure}
\centering
\includegraphics[width=0.81\columnwidth]{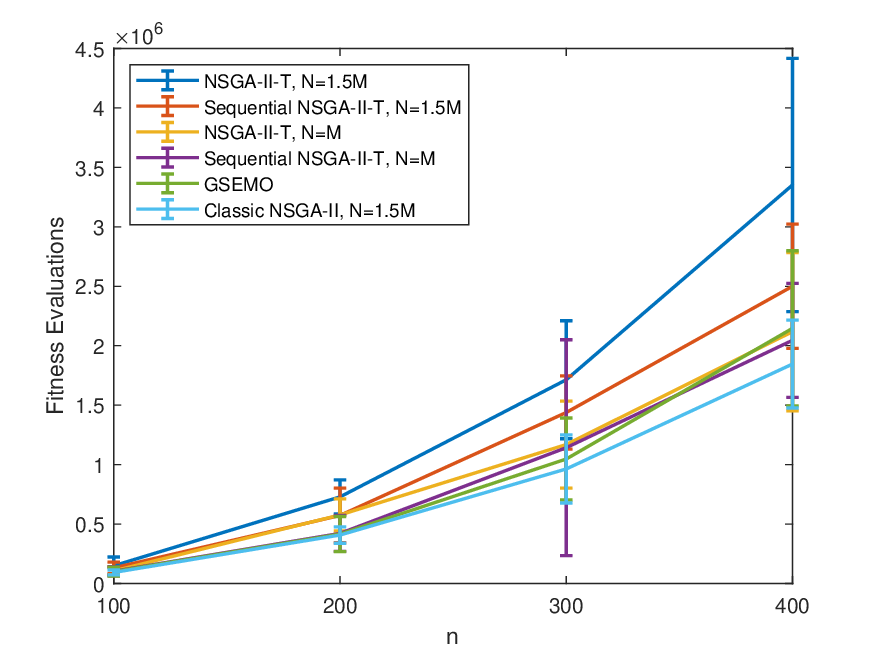}
\caption{The number of function evaluations to cover the full Pareto front for \omm.}
\label{fig:ommPF}
\end{figure}

\textbf{Approximation results:} 
To analyze who well the different \NSGA variants with small population size approximate the Pareto front, we conduct the following experiments. Note that the GSEMO cannot be used for approximative solution and is therefore not included. Following experimental settings in the only previous theoretical work~\cite{ZhengD24approx} on the approximation topic, we regard the bi-objective \omm problem with problem size $n=601$. We use the same algorithms as above (except for the GSEMO), with population sizes $N=\lceil(n+1)/2\rceil(=301),\lceil(n+1)/4\rceil(=151),\lceil(n+1)/8\rceil(=76)$. As before, we measure the approximation quality via the $\mei$. We note that the best possible  $\mei$ values are $3,5,9$ for $N=301,151,76$ respectively. 

As in~\cite{ZhengD24approx}, we regard the approximation quality in two time intervals, namely in iterations $[1..100]$ and $[3001..3100]$ after the two extreme points of the Pareto front have entered the population.

\begin{figure}
\centering
\includegraphics[width=0.81\columnwidth]{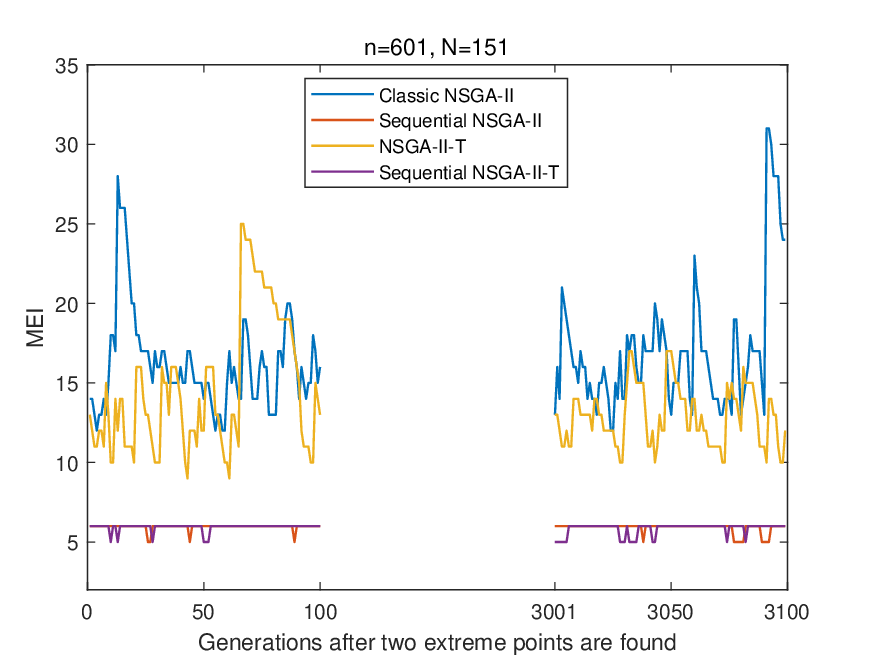}
\caption{The $\mei$ for generations $[1..100]$ and $[3001..3100]$ after the two extreme points were found (one run).}
\label{fig:omm}
\end{figure}

Figure~\ref{fig:omm} shows the $\mei$ values for the different algorithms in a single run, for reasons of space only for $N = 151$ (but the other population sizes gave a similar picture). We clearly see a much better performance of the sequential algorithms, with no significant differences between the classic and the truthful sequential \NSGA.

\section{Conclusion and Future Work}

To overcome the difficulties the \NSGA was found to have in many-objective optimization, we used the insights from several previous theoretical works, most profoundly~\cite{ZhengD23many}, to design a  truthful crowding distance for the \NSGA. Different from the original crowding distance, this new measure has the natural and desirable property that solutions with objective vector far from all others receive a large crowding distance value. The truthful crowding distances are slightly more complex to compute, but asymptotically not more complex than the non-dominated sorting step of the \NSGA. 

Via mathematical runtime analyses on several classical benchmark problems, we prove that the \NSGA with the truthful crowding distance indeed is effective in more than two objectives, admitting the same performance guarantees as previously shown for the harder to use \NSGAthree and the \SMS, which is computationally demanding due to the use of the hypervolume contribution. For the bi-objective benchmarks, for which the classic \NSGA has been analyzed, we prove the same runtime guarantees (however, only requiring the population size to be at least the size of the Pareto front). Similarly, the truthful \NSGA admits essentially the same (that is, minimally stronger) approximation guarantees as previously shown for the classic \NSGA. 

Consequently, the \NSGA with truthful crowding distance overcomes the difficulties of the classic \NSGA in many objective without that we observe disadvantages in two objectives, where the classic \NSGA has shown a very good performance.

%

\newcommand{\etalchar}[1]{$^{#1}$}

}
\end{document}